%% file: main.tex
\documentclass{article}
\usepackage[margin=1in,footskip=0.25in]{geometry} 

\usepackage[utf8]{inputenc} 
\usepackage[T1]{fontenc}    
\usepackage{hyperref}       
\usepackage{url}            
\usepackage{booktabs}       
\usepackage{amsfonts}       
\usepackage{nicefrac}       
\usepackage{microtype}      
\usepackage{amsmath}
\usepackage{amsfonts}
\usepackage{amssymb}
\usepackage{graphicx}
\usepackage{verbatim}
\usepackage{authblk}
\usepackage{graphicx}
\usepackage[square,sort,comma,numbers]{natbib}
\usepackage{doi}
\usepackage{times}
\usepackage{amsthm}
\usepackage{mathrsfs}
\usepackage{authblk}

\theoremstyle{definition} 
\newtheorem{theorem}{Theorem}

\newtheorem{lemma}{Lemma}
\newtheorem{definition}{Definition}
\usepackage{amsmath,amsfonts,amssymb,natbib, graphicx, url, algorithm2e}
\usepackage{mathtools}
\usepackage{commath}

\usepackage[utf8]{inputenc}
\usepackage[scaled=0.92]{helvet}
\usepackage{color, colortbl}
\definecolor{lightGray}{gray}{0.9}

\newcommand{\R}{\mathbb{R}}
\newcommand{\E}{\mathbb{E}}
\newcommand{\specialcell}[2][c]{%
  \begin{tabular}[#1]{@{}c@{}}#2\end{tabular}}

\newtheorem*{rep@theorem}{\rep@title}
\newcommand{\newreptheorem}[2]{%
	\newenvironment{rep#1}[1]{%
		\def\rep@title{#2 \ref{##1}}%
		\begin{rep@theorem}}%
		{\end{rep@theorem}}}
\makeatother
\newreptheorem{theorem}{Theorem}
\usepackage{multirow}
\usepackage{array}
\usepackage{url}

\title {Subspace Embeddings Under Nonlinear Transformations}

\author[1]{Aarshvi Gajjar}
\author[2]{Cameron Musco}
\tiny{
\affil[1]{UMass Amherst\\
\texttt{agajjar@umass.edu}}
\affil[2]{UMass Amherst\\
\texttt{cmusco@cs.umass.edu}}}
\date{}
\usepackage[sc,osf]{mathpazo}

\input{bib_macros}

\begin{document}
\maketitle
\begin{abstract}%
We consider low-distortion embeddings for subspaces under \emph{entrywise nonlinear transformations}. In particular we seek embeddings that preserve the norm of all vectors in a space $S = \{y: y = f(x)\text{ for }x \in Z\}$, where $Z$ is a $k$-dimensional subspace of $\R^n$ and $f(x)$ is a nonlinear activation function applied entrywise to $x$. When $f$ is the identity, and so $S$ is just a $k$-dimensional subspace, it is known that, with high probability, a random embedding into $O(k/\epsilon^2)$ dimensions preserves the norm of all $y \in S$ up to $(1\pm \epsilon)$ relative  error. Such embeddings are known as  \emph{subspace embeddings}, and have found widespread use in compressed sensing and approximation algorithms.

We give the first low-distortion embeddings for a wide class of nonlinear functions $f$. In particular, we give additive $\epsilon$ error embeddings into $O(\frac{k\log (n/\epsilon)}{\epsilon^2})$ dimensions for a class of nonlinearities that includes the popular Sigmoid SoftPlus, and Gaussian functions. We strengthen this result to give relative error embeddings under some further restrictions, which are satisfied e.g., by the Tanh, SoftSign, Exponential Linear Unit, and many other `soft' step functions and rectifying units.

Understanding embeddings for subspaces under nonlinear transformations is a key step towards extending random sketching and compressing sensing techniques for linear problems to nonlinear ones.  We discuss example applications of our results to improved bounds for compressed sensing via generative neural networks.
\end{abstract}


\section{Introduction}
Random sketching and dimensionality reduction methods are an increasingly important tool in working with massive and high-dimensional datasets \cite{bingham2001random,vempala2005random,woodruff2014sketching}. These methods attempt to very quickly compress data points into a lower-dimensional space, while still preserving important information about their structure, from which a downstream task (e.g., clustering, regression, PCA) can be solved approximately. 

    \subsection{Low-Distortion Embeddings}
Many such approaches are based around the idea of \emph{low-distortion embeddings}, dimension reducing maps which preserve the norm of all vectors in some set. %
%

\begin{definition}[Low-Distortion Embedding]\label{def:embedding} A linear map $\Pi: \R^n \rightarrow \R^{m}$ is an $(\epsilon_1,\epsilon_2)$-error embedding for $S \subseteq \R^n$ if, for all $y \in S$:
\begin{align*}
(1-\epsilon_1) \norm{y}_2 - \epsilon_2  \le \norm{\Pi y}_2 \le (1+\epsilon_1) \norm{y}_2 + \epsilon_2,
\end{align*}
where $\norm{\cdot}_2$ is the Euclidean norm. When $\epsilon_2 = 0$, we say that  $\Pi$ is an $\epsilon_1$-relative-error embedding.
\end{definition}
When the set $S$ is just a $k$-dimensional linear subspace of $\R^n$, it is  well known that letting $\Pi \in \R^{m \times n}$ be a random map (e.g., an appropriately scaled matrix with i.i.d. sub-Gaussian entries) with $m = O \left (\frac{k}{\epsilon^2}\right )$ will result in $\Pi$ being an $\epsilon$-relative error embedding for $S$ with high probability. Such an embedding is known as an \emph{oblivious subspace embedding} (OSE) since $\Pi$ can be chosen from a distribution which is oblivious to the dataset it is applied to. This is a key property e.g.,  in applications to low-memory streaming and low-communication distributed algorithms. OSE's have found a widespread application in fast algorithms for numerical linear algebra and regression \cite{sarlos2006improved,clarkson2013low,nelson2013osnap,meng2013low,woodruff2014sketching}, clustering \cite{boutsidis2010random,kmeans}, and classification \cite{paul2013random}.

%

Despite their widespread success, OSE's  only apply to \emph{linear subspaces}. Theoretical results are limited for more general sets, including natural sets arising in the application of nonlinear models such as neural networks and modern graph and work embedding methods.

\subsection{Subspace Embeddings Under Nonlinear Transformations}

In this work, we study low-distortion embeddings for \emph{subspaces under entrywise nonlinear transformations.} In particular, we study sets of the form:
\begin{align}\label{eq:setdef}
S = \{y: y = f(x)\text{ for }x \in Z\},
\end{align}
where $Z$ is a $k$-dimensional linear subspace of $\R^n$ and $f(x)$ is a nonlinear activation function applied entrywise to $x$. It is helpful to think of such a set $S$ as all possible outputs of a two layer neural network, with $k$ inputs and $n$ outputs. If $f$ is a nonlinear activation function applied to each neuron in the output layer, $W \in \R^{n \times k}$ is the weight matrix connecting the first layer to the second layer, and $x \in \R^k$ is any input, then the neural network output will be $f(W x)$. Since $Wx$ lies in a $k$-dimensional subspace (the column span of $W$), the output set is thus of the form given in \eqref{eq:setdef}.

Understanding low-distortion embeddings for the output sets of neural networks is a key theoretical tool behind recent results on compressed sensing from generative models \citep{bora2017compressed,dhar2018modeling,shah2018solving}. In particular, \cite{bora2017compressed} study the case for which $f$ is piecewise linear with $2$ pieces -- e.g., the popular ReLU activation function. In this setting, one can see that the set $S$ lies within a union of linear subspaces. Applying an OSE seperately on each of these subspaces and then taking a union bound, yields a relative  error embedding on the set $S$. \cite{bora2017compressed} also study the case for which $f$ is any Lipschitz function. This encompasses nearly all common activation functions. For such functions, one can extend the results for OSEs which are based on embedding all points in a net with bounded cardinality over the subspace. The approximation of this net is preserved under a Lipschitz transformation, and thus the same argument yields low-distortion embedding bounds for entrywise transformed subspaces. However, this approach only results in embeddings with additive (not relative) error and requires an additional restriction --  it applies to $S$ of the form:
\begin{align}\label{eq:setdefBounded}
S = \{y: y = f(x)\text{ for }x \in Z \text{ and } \norm{x}_2 \le R\},
\end{align}
where $R$ is a bound on the radius of the input set.
\subsection{Our Contributions}

We significantly extend the results on low-distortion embeddings for subspaces under nonlinear transformation. Our results, along with prior work, are summarized in Table \ref{tab:results}. Our first bound applies to a wide class of nonlinearities which (1) have a bounded second derivative and (2) approach linear asymptotes for large magnitude $x$. Such nonlinearities include for example, the Sigmoid $f(x) = \frac{1}{1+e^{-x}}$, the SoftPlus $f(x) = \ln(1+e^x)$, and the Gaussian  $f(x) = e^{-x^2}$. We show that functions of this type can be approximated to small uniform error via a piecewise linear function with a bounded number of linear regions. Applying embedding results of  \cite{bora2017compressed}  for piecewise linear functions then yields an additive error embedding for these functions. Formally:

\begin{theorem}[Additive Error Embedding]\label{thm:add-intro}
Let $S = \{y: y = f(x)\text{ for }x \in Z\}$, where $Z$ is a $k$-dimensional subspace of $\R^n$ and let $f:\mathbb{R}\to \mathbb{R}$ be a nonlinearity satisfying for constants $a,b,c, d_1,e_1,d_2,e_2$ and any $\epsilon \in (0,1]$:
\begin{enumerate}
	 \item Bounded Second Derivative: $\sup_x |f''(x) | \leq a$ and $f''$ has a finite number of discontinuities. 
    \item Linear Asymptotes:  $\forall x\geq \frac{c}{\epsilon^b}$, $|f(x)-(d_1x+e_1)|\leq\epsilon$ and $\forall x \leq -\frac{c}{\epsilon^b}$, $|f(x)-(d_2x+e_2)|\leq\epsilon$.
    \end{enumerate}
    Then, if $\Pi \in \mathbb{R}^{m\times n}$ has i.i.d entries $\Pi_{ij} \sim \mathcal{N}(0,\tfrac{1}{m})$, and $m=O\left(\tfrac{k\log(n/\epsilon_2)+\log(1/\delta)}{\epsilon_1^2}\right)$ for $\epsilon_1,\epsilon_2,\delta \in (0,1]$, with probability at least $1-\delta$, $\Pi$ is an $(\epsilon_1,\epsilon_2)$-error embedding for $S$.
\end{theorem}
For simplicity we assume $\Pi$ to be a random Gaussian embedding matrix. However, our results hold more generally for any family of random embedding matrices that yields a subspace embedding for a $k$-dimensional subspace with probability $1-\delta$ using $m = O \left ( \frac{k + \log(1/\delta)}{\epsilon^2}\right )$. See \cite{woodruff2014sketching} for a discussion of various embedding matrix distributions, many of which yield matrices that can be multiplied by much more quickly and stored in less space than a dense Gaussian embedding.

Next, we investigate relative error embeddings, which, prior to our work, were only known for linear spaces or unions of linear spaces.  These results suffice for $f$ which is piecewise linear, but not for more general functions. We give the first results for a much wider class of nonlinearities that, both satisfy the  second derivative and linear asymptote assumptions of Theorem \ref{thm:add-intro}, along with  an additional property: they are close to linear at the origin. Such nonlinearities include a large number of `soft' step functions and rectifying units, including Tanh, ArcTan, the SoftSign, the Square Nonlinearity (SQNL), and the Exponential Linear Unit (ELU). The following theorem gives an embedding for this class of functions.
\begin{theorem}[Relative Error Embedding]\label{thm:rel-intro}
Let $S = \{y: y = f(x)\text{ for }x \in Z\}$, where $Z$ is a $k$-dimensional subspace of $\R^n$ and $f:\mathbb{R}\to \mathbb{R}$ is a nonlinearity satisfying conditions (1) and (2) of Theorem \ref{thm:add-intro} along with, for some constants $g_1,g_2,g_3$:
\begin{enumerate}
\item[3.] Linear Near Origin\footnote{Note that when $f$ is bi-Lipschitz, this assumption is equivalent to $|f(y) - g_2 \cdot y| \le g_3' \cdot x^2$ for some constant $g_3'$.}: For any $y$ with $|y| \le g_1$,  $| g_2 \cdot f^{-1}(y) - y| \le g_3 \cdot y^2$.
\end{enumerate}
    Then, if $\Pi \in \mathbb{R}^{m\times n}$ has i.i.d entries $\Pi_{ij} \sim \mathcal{N}(0,\tfrac{1}{m})$, and $m=O\left(\tfrac{k\log(n/\epsilon)+\log(1/\delta)}{\epsilon^2}\right)$ for $\epsilon,\delta \in (0,1]$, with probability at least $1-\delta$, $\Pi$ is an $\epsilon$-relative-error embedding for $S$.
\end{theorem}

\begin{table}[h]
\centering
\begin{tabular}{ c | c | c | c | c }
 \textbf{Nonlinearity Class} & \textbf{Examples} & \textbf{Embedding Dim.} & \textbf{Error Type} & \textbf{Reference} \\ 
 \hline
 \specialcell{Piecewise linear \\with $t$ pieces} & \specialcell{ReLU, Binary Step \\ Leaky ReLU} & $O \left (\frac{k \log(nt)}{\epsilon^2} \right )$ & relative & \specialcell{\cite{bora2017compressed}\\ See Thm. \ref{thm:piecewise}}\\  
  \hline
 $L$-Lipschitz & Nearly all & $O \left (\frac{k \log(LR/\epsilon_2)}{\epsilon_1^2} \right )$ & \specialcell{additive, \\ input bounded\\ in radius R} & \cite{bora2017compressed}\\  
 \hline
 \rowcolor{lightGray}
 \specialcell{ $f''$ bounded,\\
linear asymptotes}  & \specialcell{Sigmoid, SoftPlus,\\ Gaussian} & $O \left (\frac{k \log(n/\epsilon_2)}{\epsilon_1^2} \right )$   & additive & Thm. \ref{thm:add-intro}  \\
\hline
 \rowcolor{lightGray}
 \specialcell{ Near-linear at origin,\\
$f''$ bounded, \\ linear asymptotes}  & \specialcell{Tanh, Arctan, SQNL\\SoftSign, ELU} & $O \left (\frac{k \log(n/\epsilon)}{\epsilon^2} \right )$   & relative & Thm. \ref{thm:rel-intro}  \\
\end{tabular}
\caption{Low-distortion embedding results (Def. \ref{def:embedding}) for $k$-dimensional subspaces under entrywise nonlinear transformations. For simplicity we hide dependences on the failure probability $\delta$ when embedding with a random linear map. Our results (highlighted in rows 3-4) significantly expand the class of nonlinearities for which low-dimensional embeddings are known and give the first  relative error results beyond piecewise linear functions.}\label{tab:results}
\end{table}

\subsection{Applications}

Our primary technical contributions are the embedding results of Theorems \ref{thm:add-intro} and \ref{thm:rel-intro}. To illustrate the usefulness of these results, in Section \ref{sec:app} we give example applications to compressed sensing from generative models \citep{bora2017compressed,shah2018solving}. In this setting, the goal is to recover $x \in \R^n$ from $m \ll n$ noisy linear measurements $y = Ax+\eta$  where $A \in \R^{m \times n}$ is a measurement matrix and $\eta \in \R^m$ is some measurement noise. 

Under the assumption that $x$ lies in some set $S$ (e.g., the set of all possible outputs of a generative neural network $G: \R^k \rightarrow \R^n$), approximate recovery up to the noise threshold $\norm{\eta}_2$ is possible when $A$ is an $(\epsilon_1,\epsilon_2)$-error embedding for $S$. Thus, our improved embedding results immediately lead to new results here, removing Lipschitzness and bounded input assumptions required by \cite{bora2017compressed} when $G$ has two layers and employs any nonlinearity satisfying Theorem \ref{thm:add-intro}.

In the important case when $G$ has $d > 2$ layers, we show how to apply our techniques to remove the bounded input assumption of \cite{bora2017compressed} for any bounded nonlinearity satisfying the assumptions of Theorem \ref{thm:add-intro}, including the Sigmoid, Gaussian, Tanh, Arctan, SoftSign, and SQNL. 



%

\subsection{Related Work}

Low-distortion embeddings are widely studied in the literature on randomized algorithms and compressed sensing. When $S$ is a \emph{finite set}, the Johnson-Lindenstrauss lemma \citep{johnson1984extensions,dasgupta1999elementary} gives that a random $\Pi \in \R^{m \times n}$ is an $\epsilon$-relative-error embedding with high probability when $m = O\left(\frac{\log |S|}{\epsilon^2}\right)$.
A majority of the work on infinite sets focuses on the case where $S$ is a linear subspace. As discussed, in this setting, many constructions for relative-error oblivious subspace embeddings (OSEs) are known. See e.g., \cite{kannan2017randomized} and \cite{woodruff2014sketching} for surveys.

The case where $S$ is the union of linear subspaces is also studied widely in the compressed sensing literature. The well known Restricted Isometry Property (RIP) is equivalent to a relative error embedding for the union of linear subspaces arising as the spans of all subsets of a fixed number of  columns of a given matrix \citep{candes2006stable,donoho2006compressed}. 

Embeddings for nonlinear spaces have been less explored.
As discussed, recent work considers low-distortion embeddings for the output sets of neural networks \citep{bora2017compressed,dhar2018modeling} with ReLU nonlinearities and under Lipschitz assumptions. We build on and significantly extend this work -- see Table \ref{tab:results} for a summary. \cite{baraniuk2009random} considers embeddings on a smooth manifold, although this is different than our nonlinear entrywise transformation setting. A number of approaches consider random projection for linear regression under various loss functions, including the Huber, Tukey, and Orlicz norm losses \citep{clarkson2014sketching,andoni2018subspace,clarkson2019dimensionality}. These methods prove low-distortion embedding results for the norms induced by these losses. This can be viewed as embedding results for the standard $\ell_1$ or $\ell_2$ norms, after applying appropriate entrywise nonlinearity, although the goal is find an embedding $\Pi \in \R^{m \times n}$ so that for $W \in \R^{n \times k}$ and all $x \in \R^k$, $\norm{f(\Pi Mx)}_2 \approx \norm{f(Mx)}_2$. This is related to but different from our goal, and requires significantly different techniques.

Finally, we note that Gordon's theorem in functional analysis \citep{gordon1988milman} gives that when $S$ is a set of unit vectors with Gaussian mean width $m = \E_{g \sim \mathcal{N}(0,1)} \sup_{x \in S} \langle g,x\rangle$, a random embedding $\Pi$ into $O \left (\frac{m^2}{\epsilon^2} \right )$ dimensions is an $\epsilon$-relative error embedding with high probability.
The Gaussian mean width is equivalent up to logarithmic factors to the Rademacher complexity of $S$, a quantity widely studied in computational learning theory \citep{shalev2014understanding}. A number of Rademacher complexity bounds are known for neural networks \citep{neyshabur2015norm,golowich2018size}, although they don't apply directly in our setting since (1) they bound the complexity of the function class corresponding to the network, rather than its output set $S$ and (2) they are parameterized by various quantities in the neural network, such as the norms  of its weight matrices. Our bounds are entirely independent of the neural network parameters, depending only on the nonlinearity used. An interesting direction for future work would be to better understand the connections between randomized dimensionality reduction for subspaces under nonlinear transformations and the work in learning theory on neural networks Rademacher complexity.

\section{Embeddings under Piecewise Linear Transformations}
We begin by showing how to extend OSE results to subspaces under piecewise linear entrywise transformations. The key idea is that such a transformation fragments the subspace into a bounded number of linear regions, each of which can be embedded with an OSE. This idea is applied e.g., by \cite{bora2017compressed} to embed ReLU networks. For completeness, we give a proof in the general case for any piecewise linear function with $t$ linear pieces. 

\begin{theorem}[Piecewise Linear Embedding]\label{thm:piecewise}
Let $Z \subseteq \mathbb{R}^n$ be a $k-$dimensional linear subspace and $f:\mathbb{R}\rightarrow \mathbb{R}$ be piecewise linear with at most $t$ pieces. Let $S = \{y:y=f(x) \text{  for } x \in Z\}$. Then if $\Pi \in \mathbb{R}^{m\times n}$ has i.i.d. entries $\Pi_{ij} \sim \mathcal{N}(0,\tfrac{1}{m})$, $m = O\left(\frac{k\log(nt)+\log(1/\delta)}{\epsilon^2}\right)$ for $\epsilon, \delta > 0$, with probability at least $1-\delta$, $\Pi$ is an $\epsilon$-relative-error embedding for $S$ (Definition \ref{def:embedding}).
\end{theorem}
We establish Theorem \ref{thm:piecewise} from the following lemma, which counts the number of $k-$dimensional linear regions in $S$. We obtain the embedding for $S$ by a union bound over these regions.   
\begin{lemma}\label{countregions}
Let $Z \subseteq \mathbb{R}^n$ be a $k-$dimensional linear subspace and $f:\mathbb{R}\rightarrow \mathbb{R}$ be piecewise linear with at most $t$  pieces. Let $S = \{y:y=f(x) \text{  for } x \in Z\}$. $S$ lies in the union of $O((tn)^k)$ $k$-dimensional linear subspaces.
\end{lemma}
\begin{proof}
Any vector $x \in Z$ can be written as $Qz$ for some $z \in \R^{k}$ where $Q \in \R^{n \times k}$ has columns spanning $Z$. Any $z \in \R^{k}$ thus corresponds to a vector $x \in S$. If we fix the pieces of $f$ that the $n$ entries of $Qz$ fall into, then $f$ simply performs a linear transformation of $Qz$, and so $x = f(Qz)$ lies in a $k$-dimensional subspace of $\R^n$. Now, each entry of $Qz$ can fall into one of $t$ pieces of $f$. Fixing which pieces it falls into splits $\R^{k}$ using $n \cdot (t-1)$ different $k-1$ dimensional hyperplanes, corresponding to the sets $\{z \in \R^k: (Qz)_i > t_j\}$ where $t_j$ is the $j^{th}$ change point of $f$.


One can show (c.f. \cite{bora2017compressed}) that $c$ hyperplanes split $\R^k$ into $O(c^k)$ regions. Plugging in $c = n \cdot (t-1)$, we have that $S$ is generated  by applying a different linear transformation to $O((tn)^k)$ regions of $\R^k$, and  thus $S$ lies in the union of $O((tn)^k)$ $k$-dimensional subspaces.
\end{proof}
\begin{proof}[Proof of Theorem \ref{thm:piecewise}]
Let $S_1,S_2 \ldots, S_{w}$ be the $w = O((tn)^k)$ linear subspaces , the union of which contains $S$. It is well known (c.f. Theorem 6 of \cite{woodruff2014sketching}) that if $\Pi\in\mathbb{R}^{n\times m}$ has independent entries $\Pi_{ij} \sim \mathcal{N}(0,\tfrac{1}{m})$ and $m = O \left (\frac{k + \log(1/\delta)}{\epsilon} \right )$, then with probability $\ge 1-\delta$, $\Pi$ is an $\epsilon$-relative-error embedding for any $k$-dimensional subspace of $\R^n$. 

Setting $\delta' = \delta/w = O(\delta/(tn)^k)$, and applying a union bound, we have that $\Pi$ is an $\epsilon$-relative-error embedding for $S_1 \cup \ldots \cup S_w \supseteq S$ with probability at least  $1-\delta$ as long as $m = O \left (\frac{k + \log(1/\delta')}{\epsilon} \right ) = O\left(\tfrac{k\log(nt)+\log(1/\delta)}{\epsilon^2}\right)$.
This completes the proof.
\end{proof} 

\section{Additive Error Embeddings}\label{sec:add}
We next show how to extend the result of Theorem \ref{thm:piecewise} to give additive error embeddings for functions that are well approximated by piecewise linear functions with a bounded number of pieces. Such functions include the popular Sigmoid activation function, the SoftPlus, and the Gaussian activation function. More generally, we give a result for any function which (1) has a bounded second derivative and (2) converges at a reasonable rate to linear asymptotes.

\begin{reptheorem}{thm:add-intro}
Let $S = \{y: y = f(x)\text{ for }x \in Z\}$, where $Z$ is a $k$-dimensional subspace of $\R^n$ and $f:\mathbb{R}\to \mathbb{R}$ is a nonlinearity satisfying for constants $a,b,c, d_1,e_1,d_2,e_2$:
\begin{enumerate}\itemsep0em 
	 \item Bounded Second Derivative: $\sup_x |f''(x) | \leq a$ and $f''$ has a finite number of discontinuities. 
    \item Linear Asymptotes: For any $\epsilon \in (0,1]$, $\forall x\geq \frac{c}{\epsilon^b}$, $|f(x)-(d_1x+e_1)|\leq\epsilon$ and $\forall x \leq -\frac{1}{\epsilon^b}$, $|f(x)-(d_2x+e_2)|\leq\epsilon$.
    \end{enumerate}
    Then, if $\Pi \in \mathbb{R}^{m\times n}$ has i.i.d entries $\Pi_{ij} \sim \mathcal{N}(0,\tfrac{1}{m})$, and $m=O\left(\tfrac{k\log(n/\epsilon_2)+\log(1/\delta)}{\epsilon_1^2}\right)$ for $\epsilon_1,\epsilon_2,\delta \in (0,1]$, with probability at least $1-\delta$, $\Pi$ is an $(\epsilon_1,\epsilon_2)$-error embedding for $S$.
\end{reptheorem}
The first assumption of bounded second derivative ensures that $f$ is well approximated by a piecewise linear function  with sufficiency small pieces. The second ensures that, outside a range of width $O(1/\epsilon^b)$ around the origin, $f(x)$ can be approximated to $\epsilon$ error via a single straight line. This is a crucial condition that applies to a large class of functions and ensures that the piecewise linear approximation has a bounded number of pieces. Formally we show:
\begin{lemma}\label{lemm:rolle} Let $f: \R \rightarrow \R$ be a function satisfying the conditions of Theorem \ref{thm:add-intro}. Then for any $\epsilon \in (0,1]$, there exists a piecewise linear function $\tilde f(x)$ with $t  = O(1/\epsilon^{b+1/2})$ pieces so that, $\forall x \in \R$, $|f(x) - \tilde f(x)| \le \epsilon$.
\end{lemma}
\begin{proof}
For $i = 0,1,\ldots,\lceil\frac{2 c}{\gamma \cdot \epsilon^b} \rceil$, let $t_i = \frac{-c}{\epsilon^b} + i \cdot \gamma$, where $\gamma$ is a stepsize we will define later. These $t_i$ divide the interval $\left [-\frac{c}{\epsilon^b},\frac{c}{\epsilon_b}\right ]$ into subintervals of length $\gamma$.
 Let $\Tilde{f}:\mathbb{R}\to\mathbb{R}$ be a piecewise linear approximation of $f$ with $\lceil\frac{2 c}{\gamma \cdot \epsilon^b} \rceil+1$ pieces defined by: 
\[
    \Tilde{f}(x)= 
\begin{cases}
    d_1x+e_1,& \text{if } x\geq \frac{c}{\epsilon^b}\\
    d_2x+e_2,& \text{if } x\leq -\frac{c}{\epsilon^b}\\
    \tiny{f(t_i)+ \frac{f(t_{i+1})-f(t_i)}{\gamma}(x-t_i)}             & \text{if } x \in \left [t_i,t_{i+1}\right ]
\end{cases}
\]
By assumption (2) of Theorem \ref{thm:add-intro} we have $|f(x) - \tilde f(x)| \le \epsilon$ for any $x \notin [-\frac{c}{\epsilon^b},\frac{c}{\epsilon^b}]$. Thus it suffices to focus on $x \in [-\frac{c}{\epsilon^b},\frac{c}{\epsilon^b}]$. Within this interval, $f$ is approximated by piecewise linear interpolation over intervals of width $\gamma$. For any $t_i$, $t_{i+1}$ and $x \in [t_i,t_{i+1}]$ it is well known that (c.f. \cite{carothers1998approximation}) Rolle's theorem yields a bound on the approximation:
\begin{align*}
|f(x) - \tilde f(x)| \le \frac{(t_{i+1}-t_{i})^2}{8} \cdot \max_{t \in [t_i,t_{i+1}]} | f''(t)| \le \frac{\gamma^2 \cdot a}{8},
\end{align*}
by our assumed upper bound of $f''(x) \le a$.
Setting $\gamma = \sqrt{\frac{8}{a}} \cdot \sqrt{\epsilon}$ we have $|f(x)-\tilde f(x)| \le \epsilon$. We note that this bound requires that $f''(x)$ is continuous on the interval $[t_i,t_{i+1}]$. Since we assume $f''(x)$ has a finite number of discontinuities, we can ensure that this is the case by placing an additional break point at each discontinuity. This will increase the number of linear pieces in $\tilde f(x)$ by just an additive constant.
The proof is now complete: $\tilde f(x)$ is a piecewise linear function with $\lceil\frac{2 c}{\gamma \cdot \epsilon^b} \rceil+1 = O \left (\frac{1}{\epsilon^{b+1/2}} \right )$ pieces with $|f(x)-\tilde f(x)| \le \epsilon$, $\forall x \in \R$.
%
\end{proof}
We Lemma \ref{lemm:rolle} in place, we now show how to extend the embedding bound of Theorem \ref{thm:piecewise} to any function that is well approximated by a  piecewise linear function.
\begin{lemma}\label{lem:th2proof}
Consider a function $f:\mathbb{R}\to\mathbb{R}$ and the set $S = \{y: y = f(x)\text{ for }x \in Z\}$ where $Z$ is a $k$-dimensional subspace of $\R^n$.
Assume that there exists piecewise linear $\tilde f:\R\rightarrow \R$ with $t$ pieces and $|f(x) - \tilde f(x)| \le \frac{\epsilon_2}{n}$ $\forall x \in \R$. Then, if $\Pi \in \mathbb{R}^{m\times n}$ has i.i.d entries $\Pi_{ij} \sim \mathcal{N}(0,\tfrac{1}{m})$, and $m=O\left(\tfrac{k\log(nt)+\log(1/\delta)}{\epsilon_1^2}\right)$, with probability at least $1-\delta$, $\Pi$ is an $(\epsilon_1,\epsilon_2)$-error embedding for $S$.
\end{lemma}
\begin{proof}
Define $\Tilde{S} = \{\Tilde{y}:\Tilde{y} = \Tilde{f}(x) \quad \text{ for } x \in Z\}$. By our approximation assumption, for all $x \in Z$, letting $y = f(x)$ and $\tilde y = \tilde f(x)$, we have:
   $ \norm{y-\Tilde{y}}_2 \leq \frac{\epsilon_2}{n} \cdot \sqrt{n} = \frac{\epsilon_2}{\sqrt{n}}$.
   Applying Theorem \ref{thm:piecewise} with parameters $\epsilon_1$ and $\delta/2$, we have that with probability at least $1-\delta/2$, $\Pi$ is an $\epsilon_1$-relative-error embedding for $\tilde S$. Additionally, it is well known (c.f. \cite{rudelson2010non}) that with probability at least $1-2 e^{-m/2} \ge 1-\delta/2$, $\Pi$'s spectral norm is bounded by $\norm{\Pi}_2 \le \frac{3\sqrt{n}}{\sqrt{m}} \le 3\sqrt{n}$.
   Assuming both events occur, which happens with probability $\ge 1-\delta$, for any $y \in S$ we have:
\begin{align}\label{eqn:2}
    \norm{\Pi y}_2 &\leq \norm{\Pi \Tilde{y}}_2 + \norm{\Pi (y-\Tilde{y})}_2\tag{triangle inequality} \nonumber\\
    &\leq (1+\epsilon_1)\norm{\Tilde{y}}_2 + \norm{\Pi}_2 \cdot \frac{\epsilon_2}{\sqrt{n}}\tag{subspace embedding} \nonumber\\
    &\leq (1+\epsilon_1) \left (\norm{y}_2 + \frac{\epsilon_2}{\sqrt{n}} \right ) + 3\epsilon_2\tag{spectral norm bound + triangle inequality}\\
    &\leq (1+\epsilon_1)\norm{y}_2 + O(\epsilon_2). \nonumber
\end{align}
Symmetrically, we can prove that $\norm{\Pi y}_2 \ge (1-\epsilon_1)\norm{y} - O(\epsilon_2)$. Adjusting constants on $m$, we have that $\Pi$ is an $(\epsilon_1,\epsilon_2)$-error embedding for $S$, completing the proof.
\end{proof}
We now combine Lemmas \ref{lemm:rolle} and \ref{lem:th2proof} to prove the additive error embedding result of Theorem \ref{thm:add-intro}.

\begin{proof}[Proof of Theorem \ref{thm:add-intro}]
By the assumptions of the theorem and Lemma \ref{lemm:rolle}, there exists piecewise linear $\tilde f: \R \rightarrow \R$ with $t= O \left(\frac{n^{b+1/2}}{\epsilon_2^{b+1/2}}\right )$ pieces and $| f(x) - \tilde f(x)| \le \frac{\epsilon_2}{n}$ for all $x \in \R$. Applying Lemma \ref{lem:th2proof}, which holds due to the existence of this $\tilde f$, we have that $\Pi$ is an $(\epsilon_1,\epsilon_2)$-error embedding for $S$ when:
$$m=O\left(\tfrac{k\log(nt)+\log(1/\delta)}{\epsilon_1^2}\right) = O\left(\tfrac{k\log(n/\epsilon_2)+\log(1/\delta)}{\epsilon_1^2}\right).$$
This completes the theorem.
\end{proof}

\subsection{Example Nonlinearities}
Many common neural network activation functions satisfy the assumptions of Theorem \ref{thm:add-intro}. Thus, the theorem provides a bound on the number of dimensions required to embed the output space of a large class of two-layer neural networks. We give some important examples below.

\paragraph{Sigmoid.} $f(x) = \frac{1}{1+e^{-x}}$.
\begin{itemize}\itemsep0em 
    \item \textit{Condition 1}:  We can compute $f''(x) = \frac{2e^{-2x}}{(1+e^{-x})^3} - \frac{e^{-x}}{(1+e^{-x})^2}$. Thus $\sup_x |f''(x)| = \sup_y |p(y)|$ where $p(x) = \frac{2 y^2}{(1+y)^3} - \frac{y}{(1+y)^2}$. We can check that this polynomial is maximized at $p(y) = \frac{1}{6\sqrt{3}}$ at $y = 2+\sqrt{3}$. Thus condition (1) of Theorem \ref{thm:add-intro} is satisfied with $a = \frac{1}{6\sqrt{3}}$. 
    \item \textit{Condition 2}: We can also check that for any $\epsilon \in (0,1]$, when $x < -\frac{1}{\epsilon} < -\ln(1/\epsilon)$, $f(x) \in [0, \epsilon)$. Similarly, when $x > \frac{1}{\epsilon} > \ln(1/\epsilon)$, $f(x) \in [\frac{1}{1+\epsilon},1] \subset [1-\epsilon,1]$. Thus, condition (2) is satisfied with $b = c = 1$, $d_1 = 1$, $d_2 = 0$,  and $e_1 = e_2 = 0$.
\end{itemize}

\paragraph{SoftPlus.} $f(x) = \ln(1+e^x)$. 
\begin{itemize}\itemsep0em 
    \item \textit{Condition 1}: We can compute $f''(x) = \frac{e^x}{(1+e^x)^2}$. Thus $\sup_x |f''(x)| = \sup_y |p(y)|$ where $p(x) = \frac{y}{(1+y)^2}$. We can check that this polynomial is maximized at $p(y) = \frac{1}{4}$ at $y = 1$. Thus condition (1) of Theorem \ref{thm:add-intro} is satisfied with $a = \frac{1}{4}$. 
    \item \textit{Condition 2}: We can also check that for any $\epsilon \in (0,1]$, when $x > \frac{1}{\epsilon} >{\ln(1/\epsilon)}$, $f(x) \ge x$ and $f(x) \le \ln((1+\epsilon)e^x) \le x + \ln(1+\epsilon) \le x + \epsilon$. Thus, $|f(x) - x| \le \epsilon$. Similarly, when $x < -\frac{1}{\epsilon} < \ln(\epsilon)$, $f(x) \ge 0$ and $f(x) \le \ln(1+\epsilon) \le \epsilon$. Thus $|f(x)| < \epsilon$. So, condition (2) is satisfied with $b = c = 1$, $d_1 = d_2 = 0$,  and $e_1 = 1$ and  $e_2 = 0$.
\end{itemize}

\paragraph{Gaussian.} $f(x) = e^{-x^2}$.
\begin{itemize}\itemsep0em 
    \item \textit{Condition 1}: We can verify that $f''(x) = e^{-x^2} (4x^2 -2)$, and has $\sup_x |f''(x)| = |f''(0)| = 2$. Thus condition (1) of Theorem \ref{thm:add-intro} is satisfied with $a = 2$.
    \item \textit{Condition 2}: We can also check that for any $\epsilon \in (0,1]$, when $|x| \ge \sqrt{\ln(1/\epsilon)} \le \frac{1}{\epsilon}$, $|f(x)|  \le \epsilon$, and thus condition (2) is satisfied with $b = c = 1$ and $d_1 = d_2 = e_1 = e_2 = 0$.
\end{itemize} 

\section{Relative Error Embeddings}\label{sec:rel}
We now show that the additive error embedding result of Theorem \ref{thm:add-intro} can be improved to relative error under the additional assumption that the nonlinearity $f$ is  close to linear near the origin. This assumption holds for a many `soft' step functions and rectifying units,  including Tanh, ArcTan, SoftSign, Square Nonlinearity (SQNL), and the Exponential Linear Unit (ELU).
\begin{reptheorem}{thm:rel-intro}
Let $S = \{y: y = f(x)\text{ for }x \in Z\}$, where $Z$ is a $k$-dimensional subspace of $\R^n$ and $f:\mathbb{R}\to \mathbb{R}$ is a nonlinearity satisfying conditions (1) and (2) of Theorem \ref{thm:add-intro} along with, for some constants $g_1,g_2,g_3$:
\begin{enumerate}
\item[3.] Linear Near Origin: For any $y$ with $|y| \le g_1$,  $| g_2 \cdot f^{-1}(y) - y| \le g_3 \cdot y^2$.
\end{enumerate}
    Then, if $\Pi \in \mathbb{R}^{m\times n}$ has i.i.d entries $\Pi_{ij} \sim \mathcal{N}(0,\tfrac{1}{m})$, and $m=O\left(\tfrac{k\log(n/\epsilon)+\log(1/\delta)}{\epsilon^2}\right)$ for $\epsilon,\delta \in (0,1]$, with probability at least $1-\delta$, $\Pi$ is an $\epsilon$-relative-error embedding for $S$.
\end{reptheorem}
\begin{proof}
Assume without loss of generality that $\epsilon < g_1$. If it is not, we can replace $\epsilon$ with $\min(g_1,\epsilon)$, and since $g_1$ is a fixed constant, this will affect the bound only by constants.
We split $S$ into two sets containing elements with relatively large norms and relatively small norms. Specifically, $S=S_L\cup S_U$ where $S_L = \{y\in S:\norm{y}_2>\epsilon/\sqrt{n}\}$ and $S_U = \{y \in S:\norm{y}_2\leq \epsilon/\sqrt{n}\}$. We then prove that with probability $1-\delta/2$, $\Pi$ is an $\epsilon$-relative-error embedding for each of $S_L$ and $S_U$. Via a union bound, this yields the theorem.

\paragraph{Case 1: $S_L$.} Since by assumption $f$ satisfies the requirements of Theorem \ref{thm:add-intro}, applying that theorem with $\epsilon_1 = \frac{\epsilon}{2}$ and $\epsilon_2 = \frac{\epsilon^2}{2\sqrt{n}}$ and  gives that, for $m = O\left(\frac{k\log(n/\epsilon) + \log(1/\delta)}{\epsilon^2}\right)$, with probability $1-\delta/2$, for all $y \in S_L$:
\begin{align*}
    \norm{\Pi y}_2 \leq (1+\tfrac{\epsilon}{2})\norm{y}_2 + \tfrac{\epsilon^2}{2\sqrt{n}} \leq (1+\epsilon)\norm{y}_2,
\end{align*}
where the second bound holds since for $y \in S_L$, $\norm{y}_2 \ge \frac{\epsilon}{\sqrt{n}}$ and thus $\frac{\epsilon^2}{2\sqrt{n}} \le \frac{\epsilon}{2} \norm{y}_2$.
Similarly, we have $\norm{\Pi y}_2 \ge (1-\epsilon) \norm{y}_2$, which completes the bound in this case.

\paragraph{Case 2: $S_U$.} We prove the theorem for $S_U$ using the fact $f$ is close to linear near the origin -- i.e., where $\norm{y}_2$ is small.
Let $\Tilde{f}(x) = g_2 \cdot x$ be a linear approximation to $f$ near the origin, i.e. for all $x$ such that $|x| < g_1, \Tilde{y} = \Tilde{f}(x)$. The approximation to $S$ thus becomes $\Tilde{S} = \{\Tilde{y}:\Tilde{y} = \Tilde{f}(x) \quad \text{ for } x \in Z\}$.
By assumption (3) of the theorem, for $y  \in S_U$, $\norm{y}_2 \le \frac{\epsilon}{\sqrt{n}}$ and thus for all $i \in \{1,2, \ldots n\}$, $|y(i)| \le \frac{\epsilon}{\sqrt{n}} < g_1$. This gives that:
\begin{align*}
|g_2 \cdot f^{-1}(y(i)) - y(i)| = | \tilde  y(i) - y(i)| \le g_3\cdot y(i)^2 \le \frac{g_3 \cdot \epsilon}{\sqrt{n}} \cdot y(i).
\end{align*}
In turn we have: 
\begin{align}\label{eq:ytildeBound}
\norm{y-\tilde y}_2 \le  \frac{g_3 \cdot \epsilon}{\sqrt{n}} \cdot \norm{y}_2.
\end{align}

Now, note that $\Tilde{S}$ is just a $k$-dimensional linear subspace. As discussed in the proof of Theorem \ref{thm:add-intro}, it is well known that for  $m = O \left (\frac{k+\log(1/\delta)}{\epsilon^2} \right  )$, with probability $\ge 1-\delta/2$, $\norm{\Pi}_2 \le 3\sqrt{n}$ and for all $\tilde y \in \tilde S$, $(1-\epsilon) \norm{\tilde y}_2 \le \norm{\Pi \tilde y}_2 \le (1+\epsilon) \norm{\tilde y}_2$ (i.e., $\Pi$ is an $\epsilon$-error subspace embedding for $\tilde S$). Along with \eqref{eq:ytildeBound}, these two conditions give that, for every $y  \in S$:
\begin{align*}
    \norm{\Pi y}_2 &\leq \norm{\Pi \Tilde{y}}_2 + \norm{\Pi(y-\tilde{y})}_2\\
    &\leq (1+\epsilon)\norm{\tilde{y}}_2 + \norm{\Pi}_2 \cdot \norm{y-\tilde{y}}_2\\
    &\leq (1+\epsilon)\norm{y}_2 + (1+\epsilon + \norm{\Pi}_2) \cdot \norm{y-\tilde{y}}_2\\
    &\leq (1+\epsilon)\norm{y}_2 + \frac{g_3 (1+\epsilon +3\sqrt{n}) \cdot \epsilon}{\sqrt{n}} \norm{y}_2  \le (1+c \epsilon) \norm{y}_2,
\end{align*}
for some constant $c$. Similarly, one can prove that $\norm{\Pi y}_2 \ge (1-c\epsilon) \norm{y}_2$. Thus, adjusting constants on $\epsilon$ by increasing $m$ by a constant gives that, with probability $1-\delta/2$, $\Pi$ is an $\epsilon$-relative-error embedding for $S$. Combined with our argument for Case 1 (the set $S_L$), this completes the proof.
\end{proof}

\subsection{Example Nonlinearities}

Many common neural network activation functions satisfy the assumptions of Theorem \ref{thm:rel-intro}. In particular, soft step functions and rectifying units (i.e., soft variants of the ReLU) often have linear asymptotes and are close to linear near the origin. We give two illustrative examples below: Tanh and ELU. Other nonlinearities, including ArcTan, SoftSign and the Square Nonlinearity (SQNL) are described in Appendix \ref{appendix:nonlinearities}.

\paragraph{Tanh (Hyperbolic Tangent).} $f(x) = \frac{e^x-e^{-x}}{e^x+e^{-x}}$
\begin{itemize}\itemsep0em 
    \item \textit{Condition 1}: We  can check that $\sup_x |f''(x)| = \frac{4}{3\sqrt{3}}$, achieved at $x = \frac{1}{2} \ln(2-\sqrt{3})$. Thus, condition (1) of Theorem \ref{thm:add-intro} is satisfied with $a = \frac{4}{3\sqrt{3}}$. 
    \item \textit{Condition 2}: For $x > \tfrac{1}{\epsilon} > \ln(1/\epsilon)$, we have $f(x) \le 1$ and $f(x) \ge \tfrac{1/\epsilon-\epsilon}{1/\epsilon+\epsilon} = \tfrac{1-\epsilon^2}{1+\epsilon^2} \ge 1-\epsilon.$ So $|f(x) -1| \le \epsilon$. Similarly, for $x < -\tfrac{1}{\epsilon} < \ln(\epsilon)$ we have $f(x) \ge -1$ and $f(x) \le \tfrac{\epsilon -1/\epsilon}{\epsilon+1/\epsilon} = \tfrac{-(1-\epsilon^2)}{1+\epsilon^2} \le 1-\epsilon.$ Thus, $|f(x) + 1| \le \epsilon$. Thus, condition (2) of Theorem \ref{thm:add-intro} is satisfied with $b = c = 1$.
    \item \textit{Condition 3}: $f^{-1}(y) = \frac{1}{2} \ln \left (\frac{1+y}{1-y} \right )$. We can check that $\frac{\left | \frac{1}{2} \ln \left (\frac{1+y}{1-y} \right ) - y \right |}{y^2} \le \frac{1}{5}$ for $y \in [-1/2,1/2]$. Thus, the final condition (3) of Theorem \ref{thm:rel-intro}  holds with $g_1 = 1/2$, $g_2 = 1$, and $g_3 = 1/5$.
\end{itemize}

\paragraph{Exponential Linear Unit (ELU).} $f(x) = \begin{cases}e^x - 1\text{ for } x \le 0 \\ x \text{ for } x \ge 0 \end{cases}$.
\begin{itemize}\itemsep0em 
    \item \textit{Condition 1}: For $x \ge 0$ we have $f''(x) = 0$. For $x \le 0$, we have $f''(x) = e^x \le 1$. Thus, 
    $\sup_x |f''(x)| \le 1$ and condition (1) of Theorem \ref{thm:add-intro} is satisfied with $a = 1$.
    \item \textit{Condition 2}: For $x>\frac{1}{\epsilon}$, we have $f(x) = x$ and thus, $|f(x) - x| = 0$. For $x<-\frac{1}{\epsilon}<-\ln(1/\epsilon)$, we have $| f(x) + 1| \le \epsilon$. Hence condition (2) of Theorem \ref{thm:add-intro} is satisfied with $b = c = 1$.
    \item \textit{Condition 3}: We have $f^{-1}(y) = \begin{cases}\ln(1+y)\text{ for } y \le 0 \\ y \text{ for } y \ge 0 \end{cases}$.\\ 
    We can check that  $\frac{|f^{-1}(y)-y|}{y^2}  \le 1$ for $y \in [-1/2,0]$ and $\frac{|f^{-1}(y)-y|}{y^2} = 0$ for $y > 0$. 
    Thus, condition (3) of Theorem \ref{thm:rel-intro} holds with $g_1 = 1/2, g_2 = 1$ and $g_3 = 1$.
\end{itemize}



\section{Application: Compressed Sensing from Generative Models}\label{sec:app}

Recently, deep generative models have become an important tool in the recovery of high-dimensional data from limited measurements using compressed sensing techniques \citep{bora2017compressed,rick2017one,shah2018solving}. 
They  have found significant success in solving linear inverse problems \citep{mccann2017convolutional}, offering a powerful alternative to the traditional structural assumption of sparsity. 

Formally, compressed sensing seeks to recover a signal $x \in \R^n$ from $m \ll n$ linear measurements, ${y} = {A}{x} + \eta$, where $A \in \R^{m \times n}$ is the measurement matrix and $\eta \in \R^m$ is some measurement noise. Recovering $x$ from $y$ requires solving this underdetermined and noisy linear system -- a task which is only possible under structural assumptions on $x$. Most commonly, in the \emph{sparse recovery} setting, it is assumed that $x$ is sparse in some basis, such as the Fourier or Wavelet basis \citep{donoho2006compressed}. Methods based on generative models instead assume that $x$ lies in the output span of some generative neural network $G:\mathbb{R}^k\to \mathbb{R}^n$. That is, $x$ lies in a low-dimensional subspace under a series of linear transformations and entrywise nonlinearities. 

\cite{bora2017compressed} extend the well-known restricted eigenvalue condition (REC) from sparse recovery, showing that, under the assumption that $x$ lies in some set $S$, as long as the objective function $\min_{x \in S} \norm{y - Ax}_2$ can be minimized to small additive error (e.g., via projected gradient descent), $x$ 
can be approximately recovered from any measurement matrix $A \in \R^{m \times n}$ satisfying the \emph{S-REC property}:
\begin{align}\label{eq:srec}\norm{A(x_1-x_2)} \geq (1-\epsilon_1)\norm{x_1-x_2} - \epsilon_2 \quad \forall x_1,x_2 \in S.
\end{align}

In turn, \cite{bora2017compressed} consider $S = \{x: x = G(z)\text{ for } z\in \R^k, \norm{z}_2 \le R \}$ -- the output span of a generative model $G$ under a bounded input restriction. They
 show that when $A \in \R^{m \times n}$ has i.i.d. $\mathcal{N}(0,1/m)$ entries, it satisfies \eqref{eq:srec} with high probability as long as $m = O \left (\frac{k\log(LR/\epsilon_2)}{\epsilon_1^2} \right )$, where $L$ is the Lipschitzness of $G$ (i.e., for any $z_1,z_2 \in \R^k$, $\norm{G(z_1)-G(z_2)} \leq L\norm{z_1-z_2}$). When $G$ uses just ReLU nonlinearities, the bounded radius and Lipschitz assumptions can be removed, $\epsilon_2 = 0$, and $m = O \left (\frac{d k\log n}{\epsilon_1^2} \right )$, where $d$ is the depth of the neural network.


\subsection{Our Results}

Our improved embedding results immediately apply to the setting of \cite{bora2017compressed}, letting us remove the dependence on the Lipschitz constant $L$ and the assumption of a bounded input $\norm{z}_2 \le R$ for two layer neural networks under the nonlinearities discussed in Sections \ref{sec:add} and \ref{sec:rel} (including the Sigmoid, Tanh, ELU, Softplus, etc.)


We employ a small modification of Theorem \ref{thm:add-intro}, which applies to the \emph{difference of two vectors} generated from a subspace under an entrywise nonlinearity. This theorem is proven essentially identically to Theorem \ref{thm:add-intro}.
\begin{theorem}[Additive Error Embedding -- Distance]\label{thm:add-diff}
Let $S = \{y: y = f(x)\text{ for }x \in Z\}$, where $Z$ is a $k$-dimensional subspace of $\R^n$ and $f:\mathbb{R}\to \mathbb{R}$ is a nonlinearity satisfying the conditions of Theorem \ref{thm:add-intro}.
    Then, if $\Pi \in \mathbb{R}^{m\times n}$ has i.i.d entries $\Pi_{ij} \sim \mathcal{N}(0,\tfrac{1}{m})$, and $m=O\left(\tfrac{k\log(n/\epsilon_2)+\log(1/\delta)}{\epsilon_1^2}\right)$ for $\epsilon_1,\epsilon_2,\delta \in (0,1]$, with probability at least $1-\delta$, for all $y_1,y_2 \in S$:
    \begin{align*}
    (1-\epsilon_1) \norm{y_1-y_2}_2 - \epsilon_2  \le \norm{\Pi (y_1-y_2)}_2 \le (1+\epsilon_1) \norm{y_1-y_2}_2 + \epsilon_2.
    \end{align*}
\end{theorem}

Now, let $G:\R^k\to\R^n$ be a two layered generative neural network with $G(z) = f(Wz)$ for some weight matrix $W \in \R^{n \times k}$ and some nonlinearity $f$ satisfying the conditions of Theorem \ref{thm:add-intro}. Let $S$ be the output set of $G$: $S = \{x \in \R^n: x = G(z)\text{ for } z \in \R^k\}$. Then Theorem \ref{thm:add-diff} implies that, when $A$ has random Gaussian entries, it satisfies the restricted eigenvalue condition of \eqref{eq:srec}, and thus, $x$ can be recovered from noisy measurements $y = Ax + \eta$. In comparison to the result of \cite{bora2017compressed}, $m=O\left(\tfrac{k\log(n/\epsilon_2)+\log(1/\delta)}{\epsilon_1^2}\right)$ has no dependence on the Lipschitzness $L$ of $G(z)$. Additionally, the bound holds under the weaker assumption that $S$ is $G$'s full output set, rather than the outputs restricted to the range of bounded diameter inputs.

\subsection{Extension to deep networks}

Our results apply to depth-2 neural networks, and an important direction for future work is to extend them to general depth-$d$ networks. In this section, we give an example of how our techniques can be applied to deeper networks.

Let $G:\R^k\to\R^n$ be a neural network with $d$ layers and $\le n$ nodes per non-input layer. The previously mentioned results of \cite{bora2017compressed} show that when $A\in \R^{n\times m}$ has i.i.d entries $A_{ij}\sim \mathcal{N}(0,\tfrac{1}{m})$, it satisfies the S-REC property of \eqref{eq:srec} for $S = \{x: x = G(z)\text{ for } z\in \R^k, \norm{z}_2 \le R \}$ and $m = O\left(\frac{k \log(LR/\epsilon_2)}{\epsilon_1^2})\right)$. We extend this result, showing how to remove the norm restriction on the representation $z$ for nonlinearities that satisfy the conditions of  Theorem \ref{thm:add-intro} and are bounded in magnitude by some constant $u$. This includes all soft step functions we consider, such as the Sigmoid, Tanh and SoftStep.


We split $G$ into the composition of two functions: $G_1:\R^k\to \R^{n}$ mapping the input layer to the second layer and $G_2:\R^{n}\to\R^n$, mapping the second layer to the output. Assume that $G_2$ is $L$-Lipschitz Note that $G_1(z) = f(W_1z)$, where $W_1$ is the weight matrix of the first layer and $f$ is the nonlinearity. Let $\Tilde{G_1}:\R^k\to\R^n$ be an approximation to $G_1$ which uses a piecewise linear approximation $\tilde f$ with $|\tilde f(x)- f(x)| \le \frac{\epsilon_2}{n L} \, \forall x$. The existence of $\tilde f$ with  $t =  O \left (\left (\frac{nL}{\epsilon_2}\right )^{b+1/2}\right )$ pieces is guarantee by Lemma \ref{lemm:rolle}. We have for any $z \in \R^k$, $\norm{G_1(z) - \tilde G_1(z)}_2 \le \frac{\epsilon_2}{\sqrt{n} L}$.

Let $\tilde G(z) = G_2(\tilde G_1(z))$. By our Lipschitzness assumption on $G_2$, for all $z$, 
\begin{align}\label{eq:lbound}\norm{G(z) - \tilde G(z)} = \norm{G_2(G_1(z)) - G_2(\tilde G_1(z))}_2 \le L \cdot \norm{G_1(z) - \tilde G_1(z)}_2 \le \frac{\epsilon_2}{\sqrt{n}}.
\end{align}
Additionally, by Lemma \ref{countregions}, the output of $\Tilde{G_1}(z)$ lies in the union of $(nt)^k$ $k$-dimensional linear subspaces. Since we assume $f(x) \le u$ for all $x$, $\tilde f(x) \le u+ \frac{\epsilon_2}{nL}$ for all $x$. Thus $\norm{\Tilde{G_1}(z)}_2 \le (u+ \frac{\epsilon_2}{nL}) \cdot \sqrt{n} = O(\sqrt{n})$.
 Thus, the output of $\tilde G(z)$ lies in the union of $t$ regions of the form $S = \{G_2(z'): z' \in Z, \norm{z'}_2 = O(\sqrt{n}) \}$, where $Z$ is a $k$-dimensional subspace. We know via the results of  \cite{bora2017compressed} and a union bound over these $t$ regions that for $m = O \left (\frac{k\log (Ln/\epsilon_2) + \log 1/\delta}{\epsilon_1^2} \right )$,
  with probability $\ge 1-\delta$, for any $\tilde x_1,\tilde x_2 \in \R^n$ in the approximate output set $\tilde S = \{\tilde G(z) : z \in \R^k \}$,
 $$\norm{A(\tilde x_1-\tilde x_2)}_2 \geq (1-\epsilon_1)\norm{\tilde x_1-\tilde x_2}_2 -\epsilon_2.$$
 For any $x_1, x_2 \in \R^n$ in the true output set $S = \{\tilde G(z) : z \in \R^k \}$ via \eqref{eq:lbound} we thus have, following the proof of Lemma \ref{lem:th2proof}:
 \begin{align*}
 \norm{A(x_1-x_2)}_2 &\ge  \norm{A(\tilde x_1-\tilde x_2)}_2 -\norm{A}_2 \cdot \frac{2\epsilon_2}{\sqrt{n}}\tag{triangle inequality}\\
 &\ge  (1-\epsilon_1)\norm{\tilde x_1-\tilde x_2}_2 -\epsilon_2 - O(\epsilon_2) \tag{$\norm{A}_2 = O(\sqrt{n})$ with high probability}\\
 &\ge (1-\epsilon_1)\norm{x_1-x_2}_2 -O(\epsilon_2) \tag{triangle inequality}
 \end{align*}
 Adjusting constants on $\epsilon_2$, this gives us the S-REC property of \eqref{eq:srec} for $S = \{G(z): z \in \R^{k}\}$ when $A$ makes $m = O \left (\frac{k\log (Ln/\epsilon_2) + \log 1/\delta}{\epsilon_1^2} \right )$ measurements. Thus, for any Lipschitz neural network using bounded linearities satisfying the assumptions of Theorem \ref{thm:add-intro}, we obtain a similar result to \cite{bora2017compressed} but without the bounded input assumption.


\subsection{Conclusions and Future Work}

Our paper makes initial steps in building a systematic understanding of randomized dimensionality reduction for subspaces under entrywise nonlinear transformations. An important next step is to extend our results to the output spaces of neural networks with $d > 2$ layers. It is possible to use an argument similar to Theorem \ref{thm:add-intro} to give some bounds here, by approximating all nonlinearities in the neural network via piecewise linear functions. However, due to compounding error at each level, $\epsilon_2$ must be set very small at the first level, leading to relatively weak embedding bounds. Understanding how to avoid this compounding error would be very  interesting.

As discussed, it would also be interesting to apply Rademacher and other complexity bounds for learning neural networks to understanding the compressibility of their output spaces and to give low-distortion embedding bounds. This would let us leverage an even richer class of tools in proving embedding bounds. 

\bibliographystyle{abbrv}
\bibliography{ref} 
\clearpage

\appendix
\section{Example Nonlinearities for Relative Error Embeddings}\label{appendix:nonlinearities}

We now give a number of other examples of nonlinearities that satisfy the assumptions of our relative error embedding result, Theorem \ref{thm:rel-intro}.

\paragraph{ArcTan.}$f(x) = \tan^{-1}(x)$
\begin{itemize}\itemsep0em 
    \item \textit{Condition 1}: . We can check that $\sup_x |f''(x)| = \frac{3\sqrt{3}}{8}$ achieved at $|x| = \frac{1}{\sqrt{3}}$. Thus, condition (1) of Theorem \ref{thm:add-intro} is satisfied with $a = \frac{3\sqrt{3}}{8}$.
    \item \textit{Condition 2}: We use a series expansion which gives that:
   $$ \tan^{-1}(x) = \begin{cases}x - \tfrac{x^3}{3}+ \tfrac{x^5}{5}   -\tfrac{x^7}{7}+.. \text{ for } |x| \le 1 \\
			\tfrac{\pi}{2}-\tfrac{1}{x} + \tfrac{1}{3x^3} - .. \text{ for } x \ge 1\\
			-\tfrac{\pi}{2}-\tfrac{1}{x} + \tfrac{1}{3x^3} - .. \text{ for } x \le -1 \end{cases}.$$

    For $x \ge \frac{1}{\epsilon}$, we thus have  $f(x)\le \frac{\pi}{2}$ and $f(x) \ge \frac{\pi}{2} - \epsilon$. Thus $|f(x) - \frac{\pi}{2}| \leq \epsilon$. Similarly, for $x \le -\frac{1}{\epsilon}$, we have $|f(x) + \frac{\pi}{2}| \leq \epsilon$. Thus, condition (2) of Theorem \ref{thm:add-intro} is satisfied with $b = c = 1$.
    \item \textit{Condition 3}: $f^{-1}(y) = \tan(y)$ for $y \in (-\frac{\pi}{2},\frac{\pi}{2})$. We can check that when $|y| \le 1$, $\frac{|tan(y) - y|}{y^2} \le \tan(1)-1 \le .56$. Thus, condition (3) of  Theorem \ref{thm:rel-intro} holds with $g_1 = 1, g_2 = 1$ and $g_3 = .56$. 
\end{itemize}

\paragraph{SoftSign.} $f(x) = \frac{x}{1+|x|}$.
\begin{itemize}\itemsep0em 
    \item \textit{Condition 1}: It can be checked that $f''(x) = \frac{2 x}{(1 + |x|)^3} - \frac{2 |x|}{x (1 + |x|)^2}$ and  $\sup_x |f''(x)| = 2$, achieved at $x = 0$. Thus, condition (1) of Theorem \ref{thm:add-intro} is satisfied with $a = 2$.
    \item \textit{Condition 2}: For $x>\frac{1}{\epsilon}$, we have $f(x) \le 1$ and $f(x) \ge 1-\frac{1}{1+x} \ge 1 - \epsilon$. Thus, $|f(x) - 1| \le \epsilon$. Similarly, for $x<-\frac{1}{\epsilon}$, we have $f(x) \ge -1$ and $f(x) \le -1 + \frac{1}{1-x} \le -1 + \epsilon$. Thus, $|f(x) + 1| \le \epsilon$. Hence condition (2) of Theorem \ref{thm:add-intro} is satisfied with $b = c= 1$.
    \item \textit{Condition 3}: We have $f^{-1}(y) = \frac{y}{1-|y|}$. It can be checked that $\frac{|f^{-1}(y)-y|}{y^2} \le 2$ when $|y| \le 1/2$. Thus, condition (3) of Theorem \ref{thm:rel-intro} holds for for $g_1 = 1/2$, $g_2 = 1$ and $g_3 = 2$.
\end{itemize}

\paragraph{Square Nonlinearity (SQNL).} Here $f(x) = \begin{cases}1 \text{ for } x \ge 2 \\ x-\frac{x^2}{4} \text{ for } x \in [0,2] \\ x+ \frac{x^2}{4}  \text{ for } x \in [-2,0] \\ -1 \text{ for } x \le 2\end{cases}$.
\begin{itemize}\itemsep0em 
    \item \textit{Condition 1}: $f''(x) = 0$ for $x \notin [-2,2]$, $f''(x) = -\frac{1}{2}$ for $x \in [0,2]$ and $f''(x) = \frac{1}{2}$ for $x \in [-2,0]$ Thus, $\sup_x |f''(x)| = \frac{1}{2} $ and so condition (1) of Theorem \ref{thm:add-intro} is satisfied with $a = \frac{1}{2}$.
    \item \textit{Condition 2}: For $x\ge \frac{1}{\epsilon}$, we have $f(x) = 1$ and hence, $|f(x) - 1| = 0$. For $x\le -\frac{1}{\epsilon}$, we have $f(x) = -1$ and hence $|f(x)+1| = 0$. Hence condition (2) of Theorem \ref{thm:add-intro} is satisfied with $b = c = 1$.
    \item \textit{Condition 3}: $f^{-1}(y) = \begin{cases} 2-2\sqrt{1-y} \text{ for } y \in [0,2] \\ -2+2\sqrt{1+y}  \text{ for } x \in [-2,0] \end{cases}$.\\
    We can check that $\frac{|f^{-1}(y)-y|}{y^2} \le 1$ for $y \in [-1/2,1/2]$, which gives that condition (3) of Theorem \ref{thm:rel-intro} holds for for $g_1 = 1/2$, $g_2 = 1$ and $g_3 = 1$.
\end{itemize}

\end{document}

%% file: bib_macros.tex
  \usepackage{nth}
  \usepackage{intcalc}